\documentclass{article}

\usepackage[preprint]{neurips_2021}
\usepackage{amsmath}
\usepackage[utf8]{inputenc}
\usepackage{amsthm}
\usepackage{makecell}
\usepackage{amssymb}
\usepackage{graphicx} 
\usepackage{subfigure}
\usepackage{booktabs} 
\usepackage{float}
\usepackage[table]{xcolor}
\usepackage{tabu}
\usepackage{bm}
\usepackage[noend]{algpseudocode}
\usepackage{multirow}
\usepackage{multicol}
\usepackage{cite}
\usepackage{algorithmicx,algorithm}
\usepackage[utf8]{inputenc} 
\usepackage[T1]{fontenc}    
\usepackage{hyperref}       
\usepackage{url}            
\usepackage{booktabs}       
\usepackage{nicefrac}       
\usepackage{microtype}      
\usepackage{xcolor} 
\usepackage{listings}

\usepackage{wrapfig}

\DeclareMathOperator*{\argmin}{arg\,min}
\title{A Gradient-based Approach for Online Robust Deep Neural Network Training with Noisy Labels}

\theoremstyle{plain}
\newtheorem{theorem}{Theorem}
\newtheorem{lemma}[theorem]{Lemma}

\theoremstyle{definition}

\newtheorem{assumption}[theorem]{Assumption}

\author{%
  Yifan Yang \\
  Department of Computer Science\\
  University of California\\
  Santa Barbara, CA, 93106\\
  \texttt{yifanyang@cs.ucsb.edu} \\
  \And
    Alec Koppel \\
  JP Morgan Chase AI Research\\
New York, NY 10017 \\
  \texttt{alec.koppel@jpmchase.com} \\
  \AND
  Zheng Zhang \\
  Department of Electrical and Computer Engineering \\
  University of California\\
  Santa Barbara, CA, 93106\\
  \texttt{zhengzhang@ece.ucsb.edu} \\
}

\begin{document}

\maketitle
\begin{abstract}
  Learning with noisy labels is an important topic for scalable training in many real-world scenarios. However, few previous research considers this problem in the online setting, where the arrival of data is streaming. In this paper, we propose a novel gradient-based approach to enable the detection of noisy labels for the online learning of model parameters, named Online Gradient-based Robust Selection (OGRS). In contrast to the previous sample selection approach for the offline training that requires the estimation of a clean ratio of the dataset before each epoch of training, OGRS can automatically select clean samples by steps of gradient update from datasets with varying clean ratios without changing the parameter setting. During the training process, the OGRS method selects clean samples at each iteration and feeds the selected sample to incrementally update the model parameters. We provide a detailed theoretical analysis to demonstrate data selection process is converging to the low-loss region of the sample space, by introducing and proving the sub-linear local Lagrangian regret of the non-convex constrained optimization problem. Experimental results show that it outperforms state-of-the-art methods in different settings.
\end{abstract}

\section{Introduction}
Online learning is a widely used learning framework for streaming data in many real-world scenarios. In recent years, online training of deep neural networks (DNNs) has garnered increased attention to enable large-scale training \cite{sahoo2017online, zappone2018online, zhan2020online}, to face the challenge of increasingly large datasets. Such a large-scale training process of DNNs, especially online DNNs training, is highly sensitive to the noisy labels in the datasets \cite{song2022learning}, which is more pronounced with the streaming and dynamically changing online data. 

The noisy label problem refers to the presence of incorrect or mislabeled annotations in a training dataset. Usually, the data samples with incorrect labels are defined as noisy data, and the correct one is called clean data. This issue has been identified as a common challenge in many datasets. For instance, researchers in \cite{northcutt2021pervasive} found 6\% label errors in the Imagenet validation set and 10\% label errors in the QuickDraw dataset. Similarly, up to 30\% label errors were found in the Google Emotions dataset \cite{chen3030} and 37\% errors in the MS COCO dataset \cite{edwin}. Label errors vary across different datasets and appear with varying probabilities of occurrence in data streams at different time slots. 

In recent years, the robustness issue of training with noisy labels has been widely studied in different research areas \cite{song2022learning, ma2022traffic, chen2023dark}. Among various approaches, sample selection methods enjoy the flexibility to support any type of deep learning architecture and do not need to maintain additional neural networks. The concept of multi-round sample selection for scalable models was first introduced in \cite{shen2019learning}, where the authors proposed an iterative training loss minimization (ITLM) method that leverages samples selected at the beginning of each training epoch. Building on this idea, INCV \cite{chen2019understanding} employs cross-validation to detect noisy training data and remove large-loss samples. O2U-Net \cite{huang2019o2u} first repeats the entire training process to collect loss statistics, then retrains the neural network from scratch only with clean samples detected. These works proposed different methods to estimate the ratio of a dataset and using sorting methods to filter out noisy data based on that ratio, but they all follow the same idea of detecting clean samples based on a fixed pre-estimated scale parameter, which is hard to be set for streaming online datasets with changing clean ratios.


In this paper, we introduce {\bf O}nline {\bf G}radient-based {\bf R}obust {\bf S}election ({\bf OGRS}), a novel gradient-based multi-iteration sample selection approach that enables the online training of DNNs with dynamically changing proportions of noisy labels. Since clean data normally produces much lower training loss compared to noisy data based on the observations in \cite{shen2019learning}, our proposed method capitalizes on the significant disparity in the gradient of the training loss at the clean and noisy data points respectively, which initially updates the data selection towards the clean region. To prevent the risk of overfitting, which may arise from the repeated selection of the same samples, we additionally propose a constraint function to mitigate the overlap of selected data. As a result, we formulate the problem as a non-convex constrained optimization problem. This structure enables our approach to dynamically adapt to varying noise proportions, thereby boosting the robustness of the online DNNs training against noisy labels.

In the realm of non-convex constrained optimization, a critical unresolved issue pertains to providing a theoretical guarantee for convergence analysis. Since the recent decade, gradient descent optimization methods have been widely used to solve a wide variety of problems, like the controls of robotic systems \cite{hu2023learning}, bayesian inference \cite{yang2023particle}, recommendation systems \cite{yang2020quantile, chen2021deeppursuit} and the training of DNNs \cite{ma2020statistical, luo2022multisource, wang2023contrastive, zhao2023tensorized,yu2022mr}. While \cite{zhang2023regrets} studied the constrained non-convex optimization problem using quadratic approximations, a straightforward analysis for this problem remains elusive due to the computational intractability of minimizing standard regret in non-convex cases. To address this challenge, we introduce a new metric for non-convex constrained optimization, termed local Lagrangian regret. We conduct a detailed theoretical analysis to validate our approach and show that a constant number of updating steps ensure our method finds a balance between the sample selection performance and the computation expense. 

In the subsequent experimental evaluation, we incrementally input data selected by the OGRS method into various online training models. These results are then benchmarked against state-of-the-art methods to demonstrate the effectiveness of our approach. In general, our main contributions are summarized below:

\begin{itemize}
    \item We introduce a novel gradient-based sample selection approach designed to facilitate effective online DNNs training with dynamically varying clean ratios.
    \item We define a new local Lagrangian regret for the non-convex optimization problem and propose an efficient algorithm that is specifically tailored to address the sample selection problem.
     \item We give a way for theoretical proof of the effectiveness and efficiency of our sample selection methods with our newly defined regret metric. 
\item We conduct experiments by simulating real-world online training cases and make comparisons between different sample selection methods.
\end{itemize}

				
				
				

\section{Related Work}

In this section, we review related work in the areas of learning with noisy labels and online DNNs training. Over the past decades, numerous deep-learning techniques have been developed to tackle the noisy label problem. These techniques are primarily grouped into five categories \cite{song2022learning}.

The first group of methods encompasses sample selection, which includes techniques of multi-network learning and multi-round learning. The multi-network learning involves the mentor network in the case of collaborative learning and multi-network learning in the case of co-training. For instance, \cite{malach2017decoupling} trained multiple DNNs simultaneously, with updates based solely on disagreements between different DNNs. On the other hand, MentorNet \cite{jiang2018mentornet} employs a mentor network to guide the training of the student network. Multi-round learning, another sample selection method, refines a selected set of clean samples at the start of each epoch \cite{shen2019learning, huang2019o2u}. Alongside methods using the small-loss trick \cite{shen2019learning, huang2019o2u} introduced in the previous section, others improve the system efficiency using a single round refinement, like \cite{wang2018iterative} and \cite{wu2020topological}. These techniques do not require the maintenance of additional DNNs, hence providing flexibility across various model architectures. In this paper, we mainly compare our results with the ITLM method in \cite{shen2019learning}, since it refines the selection set for each epoch with the small-loss trick, which is the most suitable type of related method for online robust training. The other results, like \cite{wang2018iterative}, only do a single-time selection, it naturally unsuitable for the online setting.

Additionally, certain studies have sought to design robust architectures that incorporate a noise adaptation layer at the top of the training model. Recently, such methods have been adapted to handle noisy labels \cite{xiao2015learning, chen2015webly}. Webly learning, for instance, uses the confusion matrix to initialize the weights of the noise adaptation layer \cite{chen2015webly}. In \cite{he2022robust, he2023robust}, researchers designed robust policy gradient descent methods to deal with model uncertainties. Other methodologies include robust regularization \cite{shorten2019survey, krogh1991simple}, robust loss function \cite{manwani2013noise, ghosh2017robust}, and closed-loop control \cite{chen2021towards}. Furthermore, it has been proven that designing a robust loss function for noisy data can approach the Bayesian optimal classifier \cite{natarajan2013learning}. Loss adjustment, instead of designing a new robust loss function, modifies the loss for all examples prior to the training process.

While some studies have considered online DNNs training \cite{xu2009does, sahoo2017online, sharkawy2018variable} and robust optimization \cite{he2020data, he2023data}, none have explored the problem of training with noisy labels. The inherent difficulty of online robust training lies in the dynamic nature of data streams. Existing techniques are unfit for the online scenario as it is computationally intractable to consistently modify parameters in traditional training techniques with noisy labels, such as sample selection and robust architecture. Hence, the new gradient-based sample selection method without the need for such pre-defined parameters is introduced in this paper.

\section{Framework and Preliminaries}
\begin{wrapfigure}{r}{0.5\textwidth}
\centering 
\vspace{-15mm}
\includegraphics[width=0.5\textwidth]{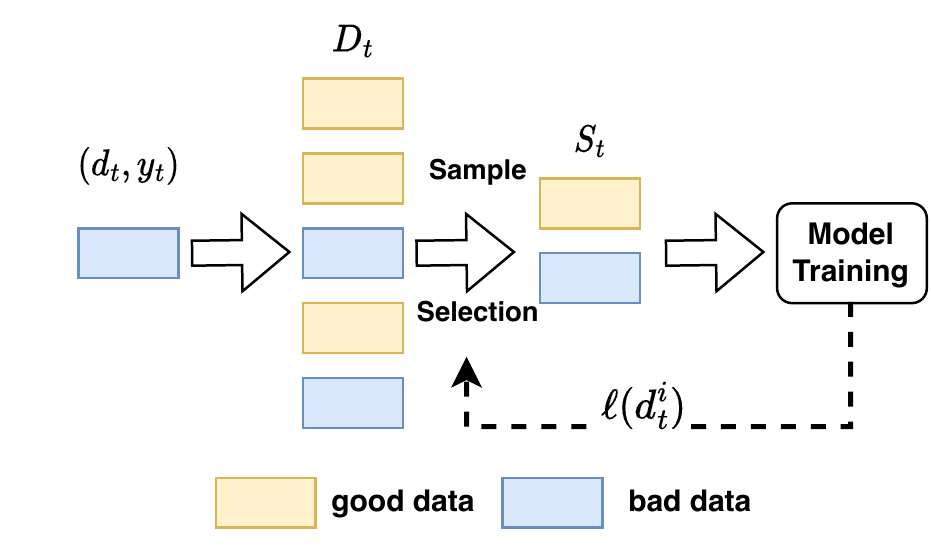}
\caption{Workflow for the OGRS}
\label{fig:workflow}
\end{wrapfigure}
In this section, we introduce the general framework for the multi-round robust training of DNNs in an online setting with time slots $t\in[1, T]$. During the training process, the online streaming data arrives at the training set $\mathcal{D}_t$ at each time slot $t$, where $\mathcal{D}_t = {(d_{1}, y_{1}), \cdots, (d_{t}, y_{t})}$ consists of a series of data pairs $(d_{t}, y_{t})$. For a single data pair $(d_{t}, y_{t})$, $d_{t}\in\mathbb{R}^d$ denotes the $d$-dimensional input features, and $y_{t}\in \mathbb{R}^C$ represents the corresponding label for $d_t$. Since we consider the training problem with noisy data, there exists a proportion of $\phi_t$ labels that are mistaken in $\mathcal{D}_t$ at the time slot $t$. 
\subsection{Problems in Directly Transferring Previous Method into Online Setting}\label{sec:itlm}
The multi-round sample selection problem considered in this paper is a traditional problem that has been widely studied since \cite{rousseeuw1984least}, which contains two parts: (1) how to select a clean sample set $S_t$ (2) how to use $S_t$ to train a neural network with parameters $\theta_t$. In Fig. \ref{fig:workflow}, we summarize the workflow of the multi-round sample selection problem for a single time slot $t$, where data pair $(d_{t}, y_{t})$ arrives at $\mathcal{D}_t$ and the sample selection algorithm select sample set $S_t$ with the loss information provided by the training model to update the model parameter $\theta_t$.

In this subsection, we will show that directly transferring some traditional sample selection methods like ITLM \cite{shen2019learning} in an online training setting didn't work well. To fit the ITLM into an online training framework, we need iteratively refines the clean set $S_t$ at the beginning of each time slot $t$. Since data with mistaken labels usually have high training loss based on observations. We sort all data samples $d\in\mathcal{D}_t$ based on their training loss $\ell_{\theta_t}(d)$ in their ITLM method. Then, the clean set $S_t$ is selected by trimming the top $1 - \hat{\phi}_t$ proportion of the sorted data list, where the pre-known clean ratio $\phi_t$ can be estimated by some techniques like crossed-validation \cite{yu2018efficient, northcutt2021confident} before starting the sample selection at each time $t$. Since there should have a total number of $t$ samples in $\mathcal{D}_t$ at time slot $t$, this process can be formulated as:
\begin{align}
    S_t \leftarrow \argmin_{S:|S| = \lfloor\phi_t t \rfloor} \sum_{d_i\in S}\ell_{\theta_t}(d_i),
\end{align}
 To update $\theta_t$ with the clean set $S_t$, the following optimization problem is solved in their training process:
\begin{align}
    \theta_{t+1}:=\argmin_\theta \ell_{\theta_t}(s_i),
\end{align}
where $s_i\in S_t$ is a batch of $K$ data samples stochastically sampled from the clean set $S_t$. Even though the online ITLM idea may work, it is obvious that directly transferring the previous method into the online setting is computationally intractable since we need to re-estimate the clean ratio at the beginning of each iteration.

\subsection{Problems for Using Traditional Local Regret Metric in Our Method}\label{sec:local}
For the sample selection part of our method, instead of using the sorting method in the ITLM, we introduce a novel sample selection algorithm based on a non-convex constrained gradient descent algorithm. To select $k$-th data sample $d_{t, k}$ in the clean sample set $S_t$ at time slot $t$, we run the constant number of gradient descent steps with iterations $i\in[1, M]$. Thus, by repeatedly running this selecting process for $K$ times, we can directly obtain a set of $S_t = \{d_{t, 1}, \cdots, d_{t, K}\}$.

To illustrate the idea of local regret metric in \cite{hazan_efficient_2017}, we focus on the updating process for selecting a single data $d_{t, k}$, which bypasses a series of decisions $d^1, \cdots, d^{M-1}, d^M_{t, k}$. To derive the local regret, we can gauge the average loss $L_{t, w}(d_t^i)$ of current $w$ time slots loss values via a sliding window, which keeps track of the algorithm performance at the point of the current decision $d^i_t$ at the iteration $i$. The definition for the $w$-local regret is established by summing up the average gradient of the local loss over a total of $M$ rounds, as described in the following equation:
\begin{align}
    R_w(M):=\sum_{i=1}^M\|\nabla L_{t, w}(d_t^i)\|^2,
\end{align}
where the averaged local loss can be calculated by $L_{t, w}(d_t^i):=\frac{1}{w}\sum_{j=0}^{w-1}\ell_{t -j}(d_t^i)$. Besides the new definition of local regret, Hazan also proposed efficient new algorithms with the local loss and gives a sublinear local regret bound in their theoretical analysis. 

Nevertheless, the original local regret metric was designed to address the unconstrained non-convex optimization problem. We cannot directly apply this setting to our sample selection problem, as we need to incorporate a constraint function to prevent the oversampling of certain samples.

\section{Algorithms}
As previously outlined, we have introduced the overall structure of our online robust training system designed for handling noisy labels. In this section, we will delve into more detailed discussions of the two main parts of the OGRS method, the sample selection process, and the model training aspects.

\subsection{Gradient-based Sample Selection}\label{sec:ss}
The key component of the OGRS is selecting the set of clean samples $S_t$ at the beginning of each time slot $t$. In order to obtain $S_t$, we repeat the updating steps for selecting a single low-loss sample $d_{t, k}$ for $K$ times, where the updating steps $d^1, d^2, \cdots, d^{M-1}$ toward the final choice $d_{t, k}$ employ a modified constrained gradient descent method. Detail of the gradient-based sample selection algorithm can be found in Alg. \ref{alg:selection}.

To deal with the non-convex loss function, we use the local loss $L_{t, w}(d)$ introduced in Sec. \ref{sec:local}. In order to avoid repeatedly selecting the same samples, a global constraint function $g_t^i$ is set as the difference between the total number of selected times for a sample $d^i_t$ and a threshold $\zeta$:
\begin{align}
    g^i_t(d^i_t) = p^i_t(d^i_t) - \zeta,
\end{align}
where $p^i_t$ indicates the total number of times that $d^i_t$ has been selected at the time slot $t$ and iteration $i$. $\zeta$ is the maximum allowance for the repeat times. As a result, we can model the sample selection part as a constrained optimization problem:
\begin{align}\label{eq:aim1}
    d_{t, k} \in \argmin_{d\in\mathcal{D}_t} L_{t, w}(d) \quad s.t. \quad g^i_t(d^i_t) \leq 0
\end{align}
To simplify the proof, we ignore the time slot $t$ and only focus on a single optimization problem with iteration $i\in[1, M]$ in the following part of this section. A widely used approach to solve the constrained optimization problem is to build the Lagrangian function that associates the loss with the constraints through a dynamically updating Lagrangian multiplier $\mu^i \in R^+$ \cite{yang4040704vflh}. To enable the non-convex constrained optimization problem, we introduce a Lagrangian function built on the local loss instead, called local Lagrangian:
\begin{align}
    \mathcal{L}^i_w(d^i, \lambda^i) = L_{ w}(d^i) +(\mu^i)^\top g^i(d^i)
\end{align}
To optimize the local Lagrangian, we introduce a modified saddle point approach, which updates decision $d^i$ in the primal update and $\mu^i$ in the dual update. Thus, the sample $d^{i+1}$ can be updated as the minimizer of the following optimization problem:
\begin{align}\label{eq:primal_update}
\min_{d \in \mathcal{D}_t} \nabla^\top L^{i}_w (d - d^i) + (\mu^{i+1})^\top g^{i}(d) + \frac{\|d - d^i\|^2}{2\alpha},
\end{align}
where $\frac{\|d - d^i\|^2}{2\alpha}$ is an added regularizer and $\alpha$ is the positive stepsize. As the current decision $d^i$ is revealed, the Lagrangian multiplier is updated based on the observation of $g^i(d^i)$ as:
\begin{align}
\mu^{i+1}=\left[\mu^i+\gamma g^i\left(d^i \right)\right]^{+}
\end{align}
To provide the theoretical guarantee of the non-convex constrained optimization algorithm, we introduce the novel local Lagrangian regret, which is defined by combining the gradient of the local loss and the constraints. The idea of the local Lagrangian regret comes from the Karush–Kuhn–Tucker (KKT) stationary conditions, as:
\begin{align}
    RL =  \|\sum_{i=1}^M \nabla L_{w}(d^i) +(\mu^i)^\top\nabla g^i(d^i)\|
\end{align}
We will show later that the proposed local Lagrangian regret can help us better understand the theory of our sample selection method. One thing that needs to be noted. Different from the traditional setting of an optimization problem that iterates the gradient descent until the current decision is close enough to the optimal decision, we only run our algorithm for constant steps. The constant steps of updates help us to reduce the over-fitting problem during the training process and are enough to guarantee a sample loss below a certain threshold. 
\subsection{Online Model Training}
After we select $S_t$ by using the proposed method. We update the model parameters $\theta_t$ by solving the following optimization problem:
\begin{align}
    \theta = \argmin_{\theta}\min_{\mathcal{S}_t}\sum_{t=1}^T\sum_{s\in\mathcal{S}_t} \ell_\theta(s)
\end{align}
These processes can be conducted by some widely used optimization methods, like the stochastic gradient descent (SGD) or the alternating direction method of multipliers (ADMM). The detail of the training process is summarized in Alg. \ref{alg:training}.

\begin{algorithm}[ht]
\caption{Gradient-based Sample Selection} 
\label{alg:selection}
\hspace*{0.02in} {\bf Input:} 
loss function $\ell_{\theta_t}(\cdot)$ of recent $w$ iterations, dataset $\mathcal{D}_t$, repeat threshold $\zeta$\\
\hspace*{0.02in} {\bf Output:} 
$S_t$
\begin{algorithmic}[1]
\For{$k\in[1, K]$}
\For{$i\in[1, M]$ epoch} 
\State Update the selected sample $d^i$ by minimizing the following optimization problem:
\begin{align}
   \min_{d \in \mathcal{D}_t} \nabla^\top L^{i}_{t, w} (d - d_t^i) + (\mu_t^{i+1})^\top g^{i}_t(d) + \frac{\|d - d_t^i\|^2}{2\alpha},
\end{align}
\State Observe the constraint violation $g_t^{i}(d_t^{i})$, where:
\begin{align}
       g^i_t(d^i_t) = p^i_t(d^i_t) - \zeta
\end{align}
\State Update the dual variable $mu_{i+1}$ by the following equation:
\begin{align}\label{eq:dual}
  \mu_t^{i+1} = [\mu_t^i + \gamma g_t^{i}(d_t^i)]^{+}
\end{align}
    \EndFor
      \State Append sample $d^{M}$ in $S_t$
    \EndFor
\State Return selected dataset $S_t$
\end{algorithmic}
\end{algorithm}

\begin{algorithm}[ht]
\caption{Online Gradient-based Robust Training} 
\label{alg:training}
\hspace*{0.02in} {\bf Input:} 
dataset $\mathcal{D}_t$\\
\hspace*{0.02in} {\bf Output:} 
model parameter $\theta_t$
\begin{algorithmic}[1]
\State Initialize the model parameters $\theta_0$
\For{$t = 1, \cdots, T$} 
\State $S_t =$ Algorithm \ref{alg:selection} $(\mathcal{D}_t, \theta_t)$
\State Update model parameter $\theta_t$ according to $S_t$
  \EndFor
\end{algorithmic}
\end{algorithm}

\section{Theoretical Analysis}

In this section, we present a theoretical analysis of the sample selection component of the OGRS, demonstrating the reliability of our method via local Lagrangian regret. We focus specifically on the process of selecting a single sample $d_{t, k}$ at time $t$. Note that the optimization strategy outlined in eq. \eqref{eq:aim1}, aimed towards the optimal, may be overly ambitious for sample selection and could potentially lead to severe overfitting issues. Consequently, we limit ourselves to a constant number of total iterations $M$ and show that constant $M$ can also give a guarantee of maximum Lagrangian residual.

Next, we first establish the regret bound in a general scenario tied to $M$. Subsequently, we delve into further exploration with a fixed setting of $M$, illustrating our method's performance. This approach is akin to setting a threshold for the maximum sample loss, a concept that aligns with previous sample selection methods that leverage the clean ratio to manage this threshold. However, we can directly use OGRS on different tasks without necessitating the configuration of specific parameters.

Before presenting the local regret bound for the Lagrangian Residual, we first enumerate some frequently employed assumptions.

\begin{assumption}\label{asp:bound}
(Bounded gradient) For every iteration $i$, both $L_{t, w}(d)$ and $g^i(d)$ are bounded and have bounded gradient, which is given as $\left|L_{t, w}(x)\right| \leq F$, $\left|\nabla L_{t, w}(x)\right| \leq F_1$, $|g^i(d^i)| \leq G $, $\left|\nabla g^i(x)\right| \leq G_1$, $\left|\nabla^2 g^i(x)\right| \leq G_2$
\end{assumption}

\begin{assumption}\label{asp:lip}
(Lipschitz continuous) The averaged loss function $L_w$ and the constraints $g^i$ are Lipschitz smooth, so its derivatives are Lipschitz continuous with constant $L_1$ and $L_2$, i.e., for two real vector $d_i, d_j \in \mathcal{D}_t$, we have:
\begin{align}
    \left|L_w(d^i)-L_w(d^j)\right| & \leq L_1 \|d_1-d_2\|\\
    \left|g^i(d^i)-g^i(d^j)\right| & \leq L_2 \|d_1 - d_2\|
\end{align}
\end{assumption}

\begin{assumption}\label{asp:boundd}
    (Bounded decision set) The sample set $\mathcal{D}_t$ is bounded, which means for some constant $D$ and any $d^i, d^j$, we have $\|d^i - d^j\| < D$
\end{assumption}

\begin{assumption}\label{asp:salter}
(Salter Condition) There exists some positive constant $\epsilon$ and an interior point $d\in\mathcal{D}$, such that $g^i(d) \leq -\epsilon \bm{I}$.
\end{assumption}

Assumptions \ref{asp:bound} are broadly employed in the non-convex optimization community. Assumption \ref{asp:lip} is crucial to ensure the validity and reasonableness of our analyses. Moreover, the Slater condition in assumption \ref{asp:salter} is instrumental in establishing the boundary of the Lagrangian multiplier. We begin by establishing a bound for the norm of Lagrangian multiplier $\|\mu^i\|$, as outlined in the following two lemmas:

\begin{lemma}\label{lemma:step}
Let $i_0$ be some arbitrary integer and $\theta$ be some real constants. For $i\in[1, M]$, the following bound holds:
\begin{align}
    & |\|\mu^{i+1}\| - \|\mu^i\|| \leq G\gamma\\ \label{eq:step}
    & |\|\mu^{i+i_0}\| - \|\mu^i\|| \leq -\frac{\epsilon i_0}{2},
\end{align}
when $\|\mu^i\| \geq \frac{\gamma^2  G^2}{\epsilon} + \frac{4\gamma F_1 D}{\epsilon} + \frac{D^2}{2\alpha\epsilon} $
\end{lemma}
\begin{proof}
     The detail of the proof can be found in Appendix \ref{app:step}.
\end{proof}

Subsequently, we establish the bound for the norm of the Lagrangian Multiplier with lemma \ref{lemma:step}.

\begin{lemma}\label{lemma:mul}
Let Assumption \ref{asp:salter} be satisfied. For the Lagrangian multiplier $\mu^i$, we can bound its norm as:
\begin{align}
    \|\mu^i\| &\leq  \frac{M^{-\frac{1}{2}}  G^2}{\epsilon} + \frac{4M^{-\frac{1}{4}} F_1 D}{\epsilon} + \frac{D^2}{2\alpha\epsilon} + i_0 G M^{-\frac{1}{4}}+ i_0\frac{8G^2M^{-\frac{1}{2}}}{\epsilon}\log[\frac{32G^2M^{-\frac{1}{2}}}{\epsilon^2}]
\end{align}
\end{lemma}
\begin{proof}
     The detail of the proof can be found in Appendix \ref{app:mul}.
\end{proof}

Finally, we proceed to present the proof of the local Lagrangian residual regret. Commencing from the stationary condition in the KKT conditions, we individually bind the components associated with the gradient of the loss function and the constraints. This procedure brings us to the subsequent Theorem:
\begin{theorem}\label{the:lag}
Let $L_{t, w}$ be the local loss function and $g^1, \cdots, g^M$ be the constraint functions in Alg. \ref{alg:selection} and all assumptions satisfied. Set $H^i = \nabla L_w + \mu^i\nabla g^i(d^i)$ and involving the results in lemma \ref{lemma:mul},  we have:

\begin{align}
     & RL(M) = \|\sum_{i=1}^M \nabla L_w + (\mu^i)^\top\nabla g^i(d^i)\| \leq \mathcal{O}(M^{\frac{1}{2}})
\end{align}

where $M$ is the maximum number of the constraints functions
\end{theorem}
\begin{proof}
     The detail of the proof can be found in Appendix \ref{app:lag}.
\end{proof}


\textbf{Remark:} From Theorem \ref{the:lag}, it is evident that the sample selection algorithm can attain a $\mathcal{O}(M^{1/2})$ local regret bound. This implies that the Lagrangian function may converge to zero when $M$ is sufficiently large. However, ensuring full convergence of the Lagrangian isn't suitable for our sample selection objectives. Instead, we restrict the maximum iterations in our algorithm to $M_{max}$, thereby setting a limit on the local Lagrangian regret with a threshold of $\mathcal{O}(\sqrt{M_{max}})$. This aligns partially with previous sample selection methods that use the task-specific estimated clean ratios as thresholds to differentiate 'good' and 'bad' samples. Notably, our OGRS method eliminates the need for such specifications. With a fixed maximum iteration setting, OGRS can handle tasks with diverse noisy training data or especially online training tasks with dynamically changing clean ratios. 

\section{Experimental Results}

In this section, we evaluate the performance of our proposed OGRS method. As other multi-round sample selection algorithms utilize a similar concept of estimating the clean ratio of the training dataset, our focus is primarily on a representative work titled ITLM (section \ref{sec:itlm}). Initially, we present intuitive synthetic results to demonstrate the performance of the OGRS method and compare it with the other methods under different parameter settings. Subsequently, we compare these methods using several real datasets, under both static and dynamically changing clean ratios.

Our experiment primarily examines random label error scenarios, wherein a certain proportion $1-\phi$ of data samples are randomly, independently, and equally likely mislabeled, where $\phi$ represents the real clean ratio of the current arrived data. Despite deep learning models' ability to automatically fit these erroneous data, our experiments reveal that our methods outperform both naive training and state-of-the-art methods when dealing with noisy datasets. Code is available at \url{https://github.com/AnonymousSubmission100/OGRS_NeurIPS/tree/main}.
\begin{figure}[t]
    \centering
    \caption{The left four figures shows the loss $\ell_{\theta_T}(d)$ after the training process, where a well-trained model should correctly distinguish clean and bad samples by assigning different loss. The right figure shows the averaged test accuracy along the time slots $t$. }
    \label{fig:syn}
    \includegraphics[width=1\linewidth]{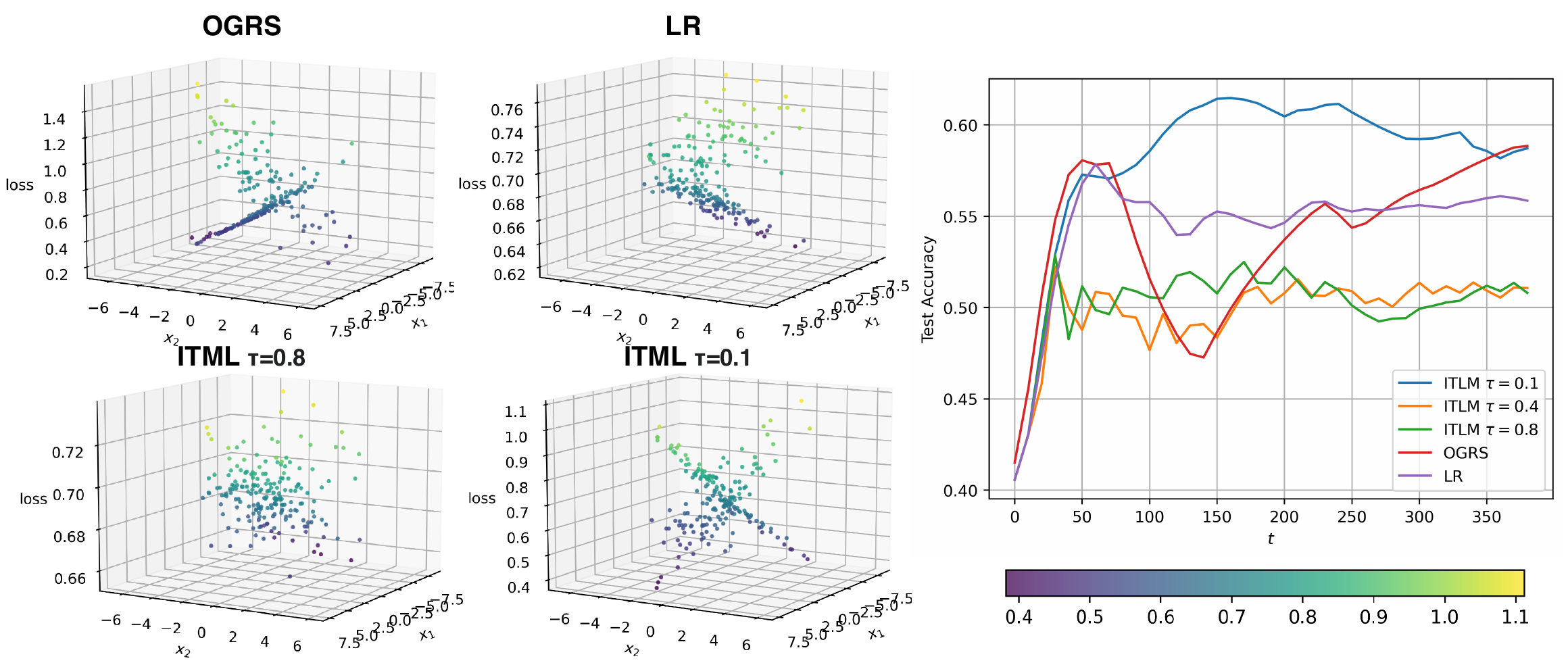}
\end{figure}
\subsection{Synthetic Experiments}
In this section, we evaluate performance on a synthetic dataset of 300 samples, delineated into two non-sensitive features $(x_1, x_2)$ and one label class $y$. The dataset, visualized in a 3D scatter plot, is partitioned into a training set of 200 samples and a test set. Samples follow a Gaussian mixture distribution $(x_1, x_2) | y=1~\mathcal{N}([1, 1]:[5,1;1, 5])$ and $(x_1, x_2) | y=0~\mathcal{N}([-1, -1]:[10,1;1, 3])$. We artificially flip $40\%$ of training labels and apply the logistic regression (LR) model.

Our Optimal Gaussian Robustness Scoring (OGRS) method is compared against naive LR and ITLM, varying the pre-estimated clean ratio $\hat{\phi}$. Unlike other methods, OGRS's parameters remain constant across tasks. Differing $\hat{\phi}$ values simulate misestimation of the clean ratio with dynamically changing data. Results are detailed in Fig. \ref{fig:syn}.

Training data loss is depicted in four left-hand figures, where the axes $(x_1, x_2)$ represent features and the third axis signifies training loss $\ell_{\theta_T}((x_1, x_2))$ post-training. ITLM with $\hat{\phi} = 0.1$ outperforms other methods, as underestimation of $\hat{\phi}$ enhances ITLM's selection accuracy. That's only happening in some simple tests like in this section. When the task becomes more and more complex, an underestimation of $\hat{\phi}$ may seriously degrade the performance as it misses a lot of valuable data points. Naive LR and ITLM with higher $\hat{\phi}$, however, struggle to differentiate good and bad data, whereas OGRS matches ITLM's performance without requiring a pre-estimated clean ratio.

Test accuracy for different methods is shown on the right, revealing that after an initial drop post the first 50 warm start rounds, OGRS quickly recovers and matches the best ITML case. In contrast, ITLM with incorrect parameters underperforms naive LR due to sample misselection and valuable data omission.

\subsection{Experiment on Real Datasets}\label{sec:real}

In this subsection, we evaluate the robustness of the OGRS method against label-randomized error in multi-dimensional, multi-error ratio datasets using different models. We employ the MNIST and CIFAR-10 datasets, testing varying clean ratios from 30\% to 70\%. A 2-layer Multi-layer Perceptron (MLP) is utilized for MNIST, while CIFAR-10 is tested with the widely used ResNet-18 model. Four training methods are compared:
\begin{itemize}
\item OGRS: Our method with \textbf{fixed parameters} across all experiments.
\item ITLM: A representative algorithm for multi-round sample selection, tested under different $\hat{\phi}$. We introduce the way to transfer ITLM to an online training setting in section \ref{sec:itlm}.
\item Naive: Directly training using all samples. 
\item Oracle: Training using only clean samples. Note that the result can only reach around 0.75 for ResNet-18 since we randomly select samples in each time slot $t$, which is different from the offline training case.
\end{itemize}

In this section, we compare various training algorithms while maintaining a constant real clean ratio $\phi$ since holding $\phi$ fixed allows for clear and plausible experimental analysis. Experiments are conducted under varying parameter configurations to highlight the advantage of OGRS, which does not necessitate a pre-estimated clean ratio. The models are trained for 10,000 rounds for experiments on both MNIST and CIFAR-10, each starting with a 500-round warm-up period using naive training.

Table \ref{tab:com} presents the results on both MNIST and CIFAR-10 datasets. In the majority of the tests, our Online Gradient-based Robust Sample selection (OGRS) method outshines the Iterative Training with Loss Minimization (ITLM) method, except for the test where $\phi=0.7$ and the pre-estimated ratio for ITLM is set as $\hat{\phi} = 0.9$. This discrepancy arises due to the setting of $\hat{\phi} = 0.9$, which ensures that ITLM samples the top 10\% of low-loss data, thereby increasing the likelihood of selecting clean data. However, this setting also causes ITLM to overlook certain observations, which in turn leads to underwhelming performance in other settings.

A vital aspect of this experiment involves testing the ITLM methods under various $\hat{\phi}$ settings. When $\hat{\phi}$ approximates the actual clean ratio $\phi$, ITLM demonstrates robustness against label errors. However, when the discrepancy between these values becomes significant, the method experiences a substantial drop in accuracy. Specifically, ITLM fails to converge in certain tasks when applied to the CIFAR-10 dataset with $\phi=0.3$ and $\phi=0.7$, especially when the estimated value $\hat{\phi}$ significantly deviates from the true value $\phi$. Such a mismatch is a common occurrence in online training settings with fluctuating clean ratios, rendering the ITLM method less suited for online training with noisy labels.

Additional experiments can be found in the Appendix.

\begin{table}[t]
	\caption{Experimental results for MNIST and CIFAR-10 datasets, utilizing different models and varying ratios $\phi$ of clean samples. To facilitate a clearer comparison of the results, the training set $\mathcal{D}_t$ for each time slot was kept constant with certain ratios of clean data. The optimal result for each setting is highlighted in bold, and 'N/A' is used to denote instances where the test did not converge to a definitive result.}
	\centering
 \label{tab:com}
	\vspace{2mm}
	\begin{tabular}{cccccccc}
		\toprule
		\multirow{2}{*}{Datasets and $\phi$} & \multicolumn{4}{c}{MNIST with MLP setup} & \multicolumn{3}{c}{CIFAR-10 with ResNet-18 setup}\\
			\cline{2-8}
		&0.9 & 0.7 & 0.5 & 0.3  & 0.9 & 0.7 & 0.5\\
		\hline
		OGRS & \textbf{0.9273} & 0.8952& \textbf{0.8886} & \textbf{0.8607} & 0.6502 & \textbf{0.5827} & \textbf{0.5149} \\
		Baseline(naive) &  0.9035 & 0.8849 & 0.8479 &  0.8098 & \textbf{0.6913} & 0.5514 & 0.3456  \\
		Oracle & 0.9208 & 0.9225 & 0.9132 & 0.8852 & 0.7502 & 0.7436 & 0.7448 \\
		ITLM ($\hat{\phi} = 0.9$)&0.9188& \textbf{0.9133}& 0.8881& 0.7742& 0.6281& 0.4466& 0.3726\\
		ITLM ($\hat{\phi} = 0.7$)& 0.7459 & 0.8325 & 0.7374 & 0.6843 & 0.5646 & N/A & 0.3105 \\
		ITLM ($\hat{\phi} = 0.5$)& 0.5654 & 0.5744 & 0.6635 & 0.5638 &0.5331 & N/A & N/A \\
            ITLM ($\hat{\phi} = 0.3$)& 0.3987 & 0.4021 & 0.3707 & 0.4003 & 0.5491 & 0.4006& N/A\\
		\bottomrule
	\end{tabular}
	\label{tab:my_label}
\end{table}

\section{Conclusion}
In this paper, we introduce a novel gradient-based sample selection method to first enable large-scale online robust training with varying proportions of noisy labels, which is a flexible method for training with noisy labels that can be added at the beginning of each iteration. We formulate the sample selection challenge as a non-convex constrained optimization problem and propose an efficient algorithm to address it. In order to give the theoretical analysis for our OGRS method, we introduce a novel metric called the local Lagrangian regret metric. We are the first to directly establish a sublinear local regret bound without resorting to the approximation of the objective function. Experimental results demonstrate that our proposed methods outperform alternatives, particularly when the pre-estimated clean ratio is hard to ascertain.

Future work can be conducted to further address the fairness issue for online training with noisy labels. Given that fairness constraints can feasibly be incorporated into the constrained optimization aspect of OGRS, which make the future research about fairness training possible.

\clearpage
\appendix
\section{Proof of Lemma \ref{lemma:step}}\label{app:step}

\begin{proof}

We give the proof in two parts. First, we prove the first inequality in eq. \eqref{eq:step}, which shows $|\|\mu^{i+1}\| - \|\mu^i\|| \leq G \gamma$

Starting from the dual updating rule in eq. \eqref{eq:dual}, we have:
\begin{align}\label{eq:part_1}
    \|\mu^{i+1} - \mu^i\| & \overset{(a)}{=} \|[\mu^i + \gamma g^i(d^i)]^+ - [\mu^i]^+\| \\
    & \leq \gamma \|g^i(d^i)\| \\
    & \overset{(b)}{\leq } G \gamma,
\end{align}
where (a) is obtained by the non-expansion property of $[\cdot]^+$ and (b) is given by Assumption \ref{asp:bound}.

What remains is to prove the second inequality $|\mu^{i+i_0}| - |\mu^{i}| \leq -\frac{\epsilon i_0}{2}$ with the hypothesis of $\|\mu^i\| \geq \frac{\gamma^2  G^2}{\epsilon} + \frac{4\gamma F_1 D}{\epsilon} + \frac{D^2}{2\alpha\epsilon}$. For $\tau\in{i, i+1, \cdots, i+i_0}$, we begin by bounding the term $|\mu^{\tau+i_0}| - |\mu^{\tau}|$. Subsequently, we make a summation over the interval $[i, i + i_0]$ to finalize the proof. For an existing saddle point $\hat{d}$, we can employ the primal updating rule in eq. \eqref{eq:primal_update}, where $d^{\tau+1}$ is the minimizer of the equation. Consequently,
\begin{align}
    &\nabla L_w^\tau(d^\tau)^\top(d^{\tau+1} - d^\tau) + \mu^{\tau+1}g(d^\tau) + \frac{|d^{\tau+1} - d^\tau|^2}{2\alpha} \\
    & \overset{(a)}{=} \nabla L_w^\tau(d^\tau)^\top(d^{\tau+1} - d^\tau) + [\mu^\tau +\gamma g^\tau(d^\tau)]^+ g(d^\tau) + \frac{|d^{\tau+1} - d^\tau|^2}{2\alpha}\\
    &\leq \nabla L_w^\tau(d^\tau)^\top(\hat{d} - d^\tau) + \mu^{\tau+1}\nabla g(\hat{d}) + \frac{|\hat{d} - d^\tau|^2 - |\hat{d} - d^{\tau+1}|^2 }{2\alpha},
\end{align}

Where $(a)$ is obtained by the dual updating rule in eq. \eqref{eq:dual}. 

Rearrange, gives:
\begin{align}
    &\mu^{\tau} g^\tau (d^\tau ) \\
&\leq \nabla L_w^\tau (d^\tau )^\top(\hat{d} - d^\tau ) - \nabla L_w^\tau (d^\tau )^\top(d^{\tau +1} - d^\tau ) + \frac{|\hat{d} - d^\tau |^2 - |\hat{d} - d^{\tau +1}|^2 }{2\alpha} \\
&- \frac{|d^{\tau +1}- d^\tau |^2}{2\alpha} + \mu^{\tau +1} g^\tau (\hat{d}) - \gamma(g^\tau (d^\tau))^2\\
&\leq \|\nabla L_w^\tau(d^\tau)\|\|\hat{d} - d^\tau\| + \|\nabla L_w^\tau (d^\tau )\|\|d^{\tau +1} - d^\tau \| + \frac{|\hat{d} - d^\tau |^2 - |\hat{d} - d^{\tau +1}|^2 }{2\alpha} + \mu^{\tau +1} g^\tau (\hat{d})\\
& \overset{(a)}{\leq} 2F_1 D -\epsilon \mu^{\tau + 1} + \frac{|\hat{d} - d^\tau |^2 - |\hat{d} - d^{\tau +1}|^2 }{2\alpha},
\end{align}
where (a) is obtained by Assumption \ref{asp:salter}.

Now, let us revisit the dual update process in eq. \eqref{eq:dual}, by taking square on the updating rule, we have:

\begin{align}
    (\mu^{\tau +1})^2 & \leq (\mu^\tau  +\gamma   g^\tau (d^\tau ) )^2\\
    & \leq (\mu^\tau )^2 + (\gamma  g^\tau (d^\tau ))^2 + 2(\mu^\tau )^\top(\gamma  g^\tau (d^\tau ))
\end{align}

Rearrange the above equality gives:

\begin{align}
    (\mu^{\tau +1})^2 - (\mu^\tau )^2 & \leq (\gamma  g^\tau (d^\tau ))^2 + 2\gamma(\mu^\tau )^\top g^\tau (d^\tau )\\
    &\overset{(a)}{\leq} \gamma^2  \|g^\tau (d^\tau )\|^2 + 2\gamma[2F_1 D -\epsilon  \mu^{\tau + 1} + \frac{|\hat{d} - d^\tau |^2 - |\hat{d} - d^{\tau +1}|^2 }{2\alpha}]\\
    & \overset{(b)}{\leq} \gamma^2  G^2 + 2\gamma[2F_1 D -\epsilon \mu^{\tau + 1} + \frac{|\hat{d} - d^\tau |^2 - |\hat{d} - d^{\tau +1}|^2 }{2\alpha}],\label{eq:part}
\end{align}
where (a) is obtained by including eq. \eqref{eq:part}. Finally, taking summation over $\tau\in\{i, i+1, \cdots, i + i_0 - 1\}$:
\begin{align} 
&(\mu^{i+i_0})^2 - (\mu^i)^2 \\
&\leq \gamma^2  G^2 i_0 + 4\gamma F_1 D i_0 -2\epsilon \sum_{\tau = i}^{i+i_0-1} \mu^{\tau + 1} + \sum_{\tau = i}^{i+i_0-1}\frac{|\hat{d} - d^\tau |^2 - |\hat{d} - d^{\tau +1}|^2 }{2\alpha}]\\
&\overset{(a)}{\leq} \gamma^2  G^2 i_0 + 4\gamma F_1 D i_0 -2\epsilon \sum_{\tau = 0}^{i_0-1} [\mu^{i+1} - i_0 G y] + \sum_{\tau = i}^{i+i_0-1}\frac{|\hat{d} - d^\tau |^2 - |\hat{d} - d^{\tau +1}|^2 }{2\alpha}]\\
& \leq \gamma^2  G^2 i_0 + 4\gamma F_1 D i_0 -2\epsilon i_0 \|\mu^{i+1}\| + \frac{i_0 D^2}{2\alpha}\\
& \overset{(b)}{\leq} \gamma^2  G^2 i_0 + 4\gamma F_1 D i_0 -2\epsilon i_0 (|\mu^i\||+ G \gamma) + \frac{i_0 D^2}{2\alpha}\\
&=  \gamma^2  G^2 i_0 + 4\gamma F_1 D i_0 -2\epsilon i_0 \|\mu^i\|-2\epsilon i_0  G \gamma + \frac{i_0 D^2}{2\alpha}
\end{align}
where (a) and (b) are given by the proof in eq. \eqref{eq:part_1}, which gives $|\|\mu^{i+1}\| - \|\mu^i\|| \leq G \gamma$.

Rearrange, give:

\begin{align}
    (\mu^{i+i_0})^2 & \leq (\mu^i)^2 - \epsilon i_0 \|\mu^i\| + (\frac{\epsilon i_0}{2})^2 + \gamma^2  G^2 i_0 + 4\gamma F_1 D i_0\\
    &-2\epsilon i_0  G \gamma + \frac{i_0 D^2}{2\alpha} - (\frac{\epsilon i_0}{2})^2 - \epsilon i_0 \|\mu^i\| \\
    & \overset{(a)}{\leq}  (\mu^i)^2 - \epsilon i_0 \|\mu^i\| + (\frac{\epsilon i_0}{2})^2\\
    & \leq (\mu^i - \frac{\epsilon i_0}{2})^2
\end{align}
where (a) is given by employing the hypothesis that $\|\mu^i\| \geq \frac{\gamma^2  G^2}{\epsilon} + \frac{4\gamma F_1 D}{\epsilon} + \frac{D^2}{2\alpha\epsilon}$, which concludes the proof $\mu^{i+i_0} - \mu^i \leq -\frac{\epsilon i_0}{2}$.

\end{proof}
\section{Proof of Lemma \ref{lemma:mul}} 
\begin{proof}

\label{app:mul}
In this section, we try to bound the norm of the Lagrangian multiplier $\|\mu^i\|$.

The proof for this lemma is a variation of lemma 5 in \cite{yu2017online}. By using the conclusion proved in Lemma \ref{lemma:step}, we have:
\begin{align}
    & |\|\mu^{i+1}\| - \|\mu^i\|| \leq G\gamma\\ 
    & |\|\mu^{i+i_0}\| - \|\mu^i\|| \leq -\frac{\epsilon i_0}{2},
\end{align}
when $\|\mu^i\| \geq \frac{\gamma^2  G^2}{\epsilon} + \frac{4\gamma F_1 D}{\epsilon} + \frac{D^2}{2\alpha\epsilon} $.

Apply the lemma 5 of \cite{yu2017online} and we can obtain the bound for $|\mu^i|$ as:
\begin{align}
    |\mu^i| & \leq \frac{\gamma^2  G^2}{\epsilon} + \frac{4\gamma F_1 D}{\epsilon} + \frac{D^2}{2\alpha\epsilon} + i_0 G\gamma+ i_0\frac{8G^2\gamma^2}{\epsilon}\log[\frac{32G^2\gamma^2}{\epsilon^2}]\\
    &\overset{(a)}{\leq} \frac{M^{-\frac{1}{2}}  G^2}{\epsilon} + \frac{4M^{-\frac{1}{4}} F_1 D}{\epsilon} + \frac{D^2}{2\alpha\epsilon} + i_0 G M^{-\frac{1}{4}}+ i_0\frac{8G^2M^{-\frac{1}{2}}}{\epsilon}\log[\frac{32G^2M^{-\frac{1}{2}}}{\epsilon^2}]\\
    & \leq \mathcal{O}(M^{-\frac{1}{2}})
\end{align}
where (a) is obtained by taking a shrinking stepsize $\gamma = M^{-1/4}$.
\end{proof}

\section{Proof of Theorem \ref{the:lag}}\label{app:lag}
\begin{proof}

We begin the proof by giving a lemma that shows the updating rules in eq. \eqref{eq:primal_update} can be formulated as a projection gradient descent step:
\begin{lemma}\label{lemma:proj}
    Set $H^i = \nabla L_w + (\mu^i)^\top\nabla g^i(d^i)$, the update rule in eq. \eqref{eq:primal_update} can be given by:
    \begin{align}\label{eq:proj}
        d^{i+1} = \mathcal{P}_{\mathcal{D}_t}[d^i - \alpha H^i]
        \end{align}
\end{lemma}

\begin{proof}
    By the definition of the projection, eq. \eqref{eq:proj} is equal to the following equation:
    \begin{align}
        &\min_{d\in\mathcal{D}_t}\|d - (d^i - \alpha H^i )\|^2\\
         &= \min_{d\in\mathcal{D}_t}\|(d - d^i) +\alpha(  \nabla L_w + \mu^i\nabla g^i(d^i) )\|^2\\
         &\overset{(a)}{\leq} \min_{d\in\mathcal{D}_t}\|(d - d^i)\|^2 +2 \alpha( \nabla L_w + (\mu^i)^\top\nabla g^{i}(d^i))^\top (d - d^i)\\
         &\overset{(b)}{\leq}\min_{d\in\mathcal{D}_t} \frac{1}{2\alpha}\|(d - d^i)\|^2  + \nabla L^{i}_w (d - d^i) + g^{i}(d^i)^\top (d - d^i)\\
         & \overset{(c)}{\leq} \min_{d\in\mathcal{D}_t} \frac{1}{2\alpha}\|(d - d^i)\|^2  + \nabla L^{i}_w (d - d^i) + 
         (\mu^i)^\top (g^{i}(d)  - g^{i}(d^i))\\
         & \leq \min_{d\in\mathcal{D}_t} \frac{1}{2\alpha}\|(d - d^i)\|^2  + \nabla L^{i}_w (d - d^i) + 
         (\mu^i)^\top g^{i}(d),
    \end{align}
    where (a) is obtained by ignoring the term $(H^{i})^2$, (b) is given by multiple $1/2\alpha$ on all terms and (c) follows the convexity of the constraints functions.
\end{proof}

 By using the result in Lemma. \ref{lemma:proj}, we have:
 \begin{align}
      d^{i+1} &= \min_{d \in \mathcal{D}_t} \nabla^\top L^{i}_w (d - d^i) + (\mu^i)^\top g^{i}(d) + \frac{|d - d^i|^2}{2\alpha}\\
      &= \mathcal{P}_{\mathcal{D}_t}[d^i - \alpha(\nabla L_w(d^i)+ (\mu^i)^\top\nabla g^i(d))] \label{eq:first}
 \end{align}

Next, we establish the bound for the norm of the term $|H^i| = |\nabla L_w + \mu^i\nabla g^i(d^i)|$, which is associated with both local regret and local constraints. We begin with KKT stationary condition for the projected gradient descent step found in eq. \eqref{eq:first}. This can be expressed as follows:
\begin{align}\label{kkt}
    0 \in \nabla L_w(d^i)+(\mu^{i})^\top \nabla g^i(d^i)
\end{align}

Rewrite (\ref{kkt}) with $H^{i+1}$ as:

\begin{align}
     H^{i+1} - H^{i+1} + \nabla L_w(d^i)+\mu^{i} \nabla g^i(d^i)= 0
\end{align}

Moving one $H^i$ term to another side and combining another $H^i$ term with the other parts. By taking summation over $t^i\in[1, M]$, we give:
\begin{align}
    -\sum_{i=1}^M H^i & \leq \sum_{i=1}^M (\nabla L_w(d^i) - \nabla L^{i+1}(d^{i+1})) + \sum_{i=1}^M \mu^{i} (\nabla g^i(d^i) - \nabla g^{i+1}(d^{i+1}))\\
     & = (\nabla L^1(d^1) - \nabla L^{M+1}(d^{M+1})) + \sum_{i=1}^M \mu^{i} (\nabla g^i(d^i) - \nabla g^{i+1}(d^{i+1}))
\end{align}

Taking norm on both sides:
\begin{align}
    \|\sum_{i=1}^M H^i\| & \leq 
      \|\nabla L^1(d^1)\| +\| \nabla L^{M+1}(d^{M+1})\| + \sum_{i=1}^M \|\mu^{i}\| \|\nabla g^i(d^i) - \nabla g^{i+1}(d^{i+1})\| \\
      & \overset{(a)}{\leq} 2 F_1 + \sum_{i=1}^M \|\mu^{i}\| \|\nabla g^i(d^i) - \nabla g^{i+1}(d^{i+1})\| \\ 
      & \overset{(b)}{\leq} 2 F_1 + L_2\sum_{i=1}^M \|\mu^{i}\| \|d^i - d^{i+1}\|\\
      & \overset{(c)}{\leq} 2 F_1 + L_2D M(\frac{M^{-\frac{1}{2}}  G^2}{\epsilon} + \frac{4M^{-\frac{1}{4}} F_1 D}{\epsilon} + \frac{D^2}{2\alpha\epsilon} \\
      & + i_0 G M^{-\frac{1}{4}}+ i_0\frac{8G^2M^{-\frac{1}{2}}}{\epsilon}\log[\frac{32G^2M^{-\frac{1}{2}}}{\epsilon^2}])\\
      & \leq \mathcal{O}(M^{\frac{1}{2}})
\end{align}
1where (a) is obtained by Assumption \ref{asp:bound}, (b) is using the $L_2$-smoothness of the function $g^i(d^i)$ in Assumption \ref{asp:lip}, (c) is using the Lemma \ref{lemma:mul} and Assumption \ref{asp:bound} again.
\end{proof}

\section{Experiments with Dynamic Changing Dataset}
In this section, we compare the proposed OGRS method with the previously sorting-based ITLM methods with a dynamically changing dataset. We modified the CIFAR-10 dataset to conduct the experiment, by randomly selecting 20000 data samples and dividing them into four parts of subsets, each with 5000 data samples. Then, we randomized labels of a certain clean ratio of data in each part of the subset. In this experiment, we set the clean ratios $\phi$ for the four subsets as 0.1, 0.3, 0.2, and 0.15. We feed the same subset into the robust training system every 5000 iterations. For example, we use the subset with $\phi=0.1$ for the training time slot $t\in[1, 5000]$ and use the subset with $\phi=0.3$ for the training time slot $t\in[5000, 10000]$, etc. We record test accuracy every 5000 rounds and the test results are summarized in Fig. \ref{fig:dyc}.

During the training process, it is important to note that all parameters for each method remained fixed at their initial values. As a result, both methods exhibited performance that was inferior to the fixed clean ratio case. This outcome can be attributed to the fact that the model was unable to fully adapt and train effectively under dynamically changing data conditions. In Section \ref{sec:real}, we also highlighted a similar scenario where the online setting's dynamic nature limited the selection of data to random sampling rather than dividing the dataset into predefined batches. This constraint also contributed to the suboptimal performance of all online sample selection systems. Despite the challenges posed by the dynamically changing clean ratio, it is worth noting that OGRS demonstrated a higher level of stability and achieved better test accuracy compared to ITLM. This observation indicates the adaptability and resilience of the OGRS method in the face of fluctuating clean ratios.

\begin{figure}[!h]
    \centering
    \caption{Comparison between the OGRS and ITLM with dynamically changing datasets.}
    \label{fig:dyc}
    \includegraphics[width=0.5\linewidth]{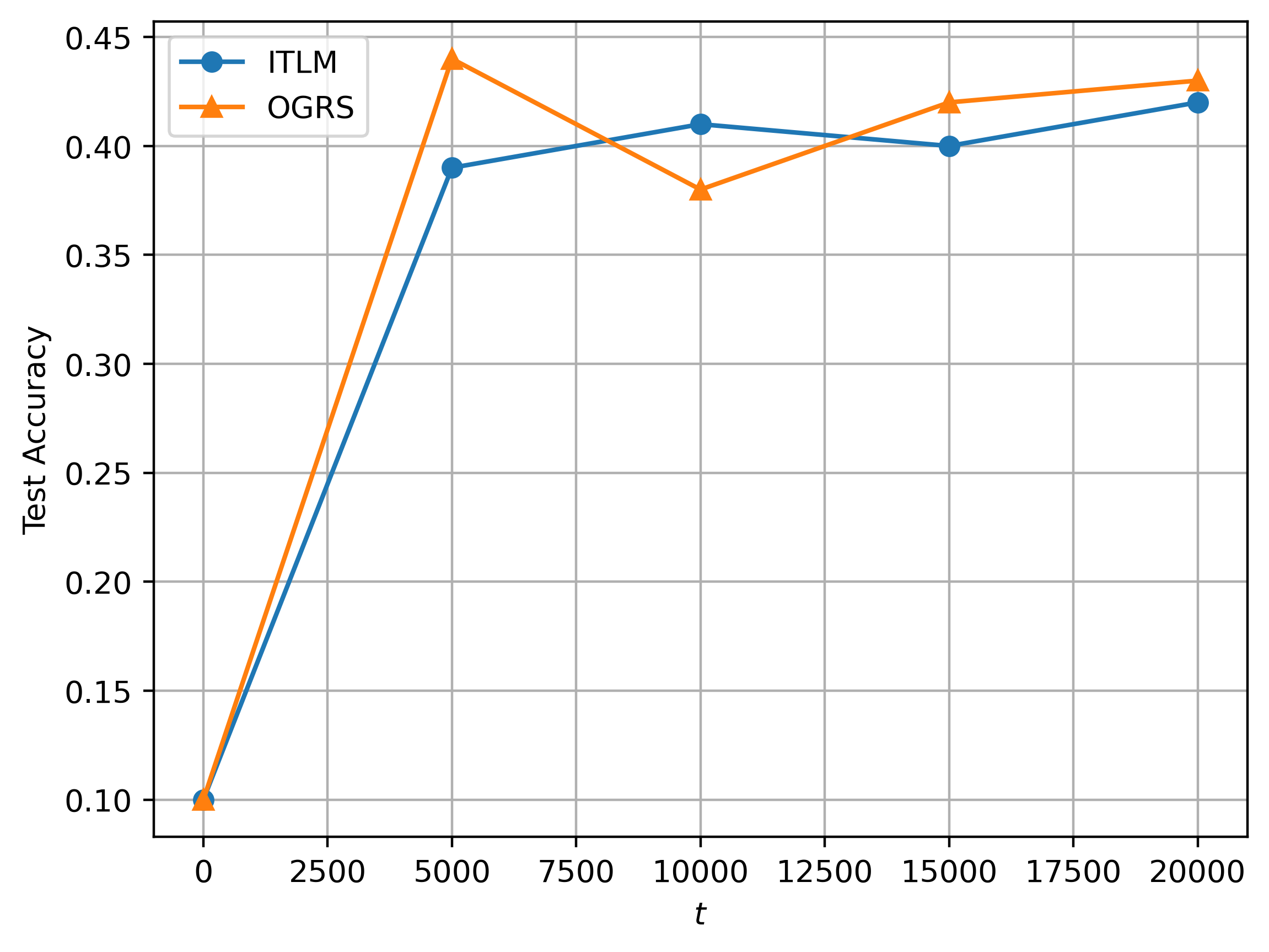}
    \vspace{-10mm}
\end{figure}


\begin{thebibliography}{10}

\bibitem{sahoo2017online}
Doyen Sahoo, Quang Pham, Jing Lu, and Steven~CH Hoi.
\newblock Online deep learning: Learning deep neural networks on the fly.
\newblock {\em arXiv preprint arXiv:1711.03705}, 2017.

\bibitem{zappone2018online}
Alessio Zappone, M{\'e}rouane Debbah, and Zwi Altman.
\newblock Online energy-efficient power control in wireless networks by deep
  neural networks.
\newblock In {\em 2018 IEEE 19th International Workshop on Signal Processing
  Advances in Wireless Communications (SPAWC)}, pages 1--5. IEEE, 2018.

\bibitem{zhan2020online}
Xiaohang Zhan, Jiahao Xie, Ziwei Liu, Yew-Soon Ong, and Chen~Change Loy.
\newblock Online deep clustering for unsupervised representation learning.
\newblock In {\em Proceedings of the IEEE/CVF conference on computer vision and
  pattern recognition}, pages 6688--6697, 2020.

\bibitem{song2022learning}
Hwanjun Song, Minseok Kim, Dongmin Park, Yooju Shin, and Jae-Gil Lee.
\newblock Learning from noisy labels with deep neural networks: A survey.
\newblock {\em IEEE Transactions on Neural Networks and Learning Systems},
  2022.

\bibitem{northcutt2021pervasive}
Curtis~G Northcutt, Anish Athalye, and Jonas Mueller.
\newblock Pervasive label errors in test sets destabilize machine learning
  benchmarks.
\newblock {\em arXiv preprint arXiv:2103.14749}, 2021.

\bibitem{chen3030}
Edwin Chen.
\newblock 30\% of google’s emotions dataset is mislabeled.
\newblock {\em URL https://www. surgehq.
  ai/blog/30-percent-of-googles-reddit-emotions-dataset-is-mislabeled}, 25,
  2022.

\bibitem{edwin}
Edwin Murdoch.
\newblock How i found nearly 300,000 errors in ms coco.
\newblock {\em URL
  $https://medium.com/@jamie_34747/how-i-found-nearly-300-000-errors-in-ms-coco-79d382edf22b$},
  25, 2022.

\bibitem{ma2022traffic}
Xiaobo Ma.
\newblock {\em Traffic Performance Evaluation Using Statistical and Machine
  Learning Methods}.
\newblock PhD thesis, The University of Arizona, 2022.

\bibitem{chen2023dark}
Ziheng Chen, Fabrizio Silvestri, Jia Wang, Yongfeng Zhang, and Gabriele
  Tolomei.
\newblock The dark side of explanations: Poisoning recommender systems with
  counterfactual examples.
\newblock {\em arXiv preprint arXiv:2305.00574}, 2023.

\bibitem{shen2019learning}
Yanyao Shen and Sujay Sanghavi.
\newblock Learning with bad training data via iterative trimmed loss
  minimization.
\newblock In {\em International Conference on Machine Learning}, pages
  5739--5748. PMLR, 2019.

\bibitem{chen2019understanding}
Pengfei Chen, Ben~Ben Liao, Guangyong Chen, and Shengyu Zhang.
\newblock Understanding and utilizing deep neural networks trained with noisy
  labels.
\newblock In {\em International Conference on Machine Learning}, pages
  1062--1070. PMLR, 2019.

\bibitem{huang2019o2u}
Jinchi Huang, Lie Qu, Rongfei Jia, and Binqiang Zhao.
\newblock O2u-net: A simple noisy label detection approach for deep neural
  networks.
\newblock In {\em Proceedings of the IEEE/CVF international conference on
  computer vision}, pages 3326--3334, 2019.

\bibitem{hu2023learning}
Xiaohai Hu, Aparajit Venkatesh, Guiliang Zheng, and Xu~Chen.
\newblock Learning to detect slip through tactile measures of the contact force
  field and its entropy.
\newblock {\em arXiv preprint arXiv:2303.00935}, 2023.

\bibitem{yang2023particle}
Yifan Yang, Chang Liu, and Zheng Zhang.
\newblock Particle-based online bayesian sampling.
\newblock {\em arXiv preprint arXiv:2302.14796}, 2023.

\bibitem{yang2020quantile}
Yifan Yang, Jie Xu, Zichuan Xu, Pan Zhou, and Tie Qiu.
\newblock Quantile context-aware social iot service big data recommendation
  with d2d communication.
\newblock {\em IEEE Internet of Things Journal}, 7(6):5533--5548, 2020.

\bibitem{chen2022grease}
Ziheng Chen, Fabrizio Silvestri, Jia Wang, Yongfeng Zhang, Zhenhua Huang,
  Hongshik Ahn, and Gabriele Tolomei.
\newblock Grease: Generate factual and counterfactual explanations for
  gnn-based recommendations.
\newblock {\em arXiv preprint arXiv:2208.04222}, 2022.

\bibitem{chen2021deeppursuit}
Ziheng Chen, Sichen Zhong, Jianshu Chen, and Yue Zhao.
\newblock Deeppursuit: Uniting classical wisdom and deep rl for sparse
  recovery.
\newblock In {\em 2021 55th Asilomar Conference on Signals, Systems, and
  Computers}, pages 1361--1366. IEEE, 2021.

\bibitem{ma2020statistical}
Xiaobo Ma, Abolfazl Karimpour, and Yao-Jan Wu.
\newblock Statistical evaluation of data requirement for ramp metering
  performance assessment.
\newblock {\em Transportation Research Part A: Policy and Practice},
  141:248--261, 2020.

\bibitem{luo2022multisource}
Xiaoling Luo, Xiaobo Ma, Matthew Munden, Yao-Jan Wu, and Yangsheng Jiang.
\newblock A multisource data approach for estimating vehicle queue length at
  metered on-ramps.
\newblock {\em Journal of Transportation Engineering, Part A: Systems},
  148(2):04021117, 2022.

\bibitem{wang2023contrastive}
Zihu Wang, Yu~Wang, Hanbin Hu, and Peng Li.
\newblock Contrastive learning with consistent representations.
\newblock {\em arXiv preprint arXiv:2302.01541}, 2023.

\bibitem{zhao2023tensorized}
Yequan Zhao, Xian Xiao, Geza Kurczveil, Raymond~G Beausoleil, and Zheng Zhang.
\newblock Tensorized optical multimodal fusion network.
\newblock {\em arXiv preprint arXiv:2302.08744}, 2023.

\bibitem{zhang2023regrets}
Liwei Zhang, Haoyang Liu, and Xiantao Xiao.
\newblock Regrets of proximal method of multipliers for online non-convex
  optimization with long term constraints.
\newblock {\em Journal of Global Optimization}, 85(1):61--80, 2023.

\bibitem{malach2017decoupling}
Eran Malach and Shai Shalev-Shwartz.
\newblock Decoupling" when to update" from" how to update".
\newblock {\em Advances in neural information processing systems}, 30, 2017.

\bibitem{jiang2018mentornet}
Lu~Jiang, Zhengyuan Zhou, Thomas Leung, Li-Jia Li, and Li~Fei-Fei.
\newblock Mentornet: Learning data-driven curriculum for very deep neural
  networks on corrupted labels.
\newblock In {\em International conference on machine learning}, pages
  2304--2313. PMLR, 2018.

\bibitem{wang2018iterative}
Yisen Wang, Weiyang Liu, Xingjun Ma, James Bailey, Hongyuan Zha, Le~Song, and
  Shu-Tao Xia.
\newblock Iterative learning with open-set noisy labels.
\newblock In {\em Proceedings of the IEEE conference on computer vision and
  pattern recognition}, pages 8688--8696, 2018.

\bibitem{wu2020topological}
Pengxiang Wu, Songzhu Zheng, Mayank Goswami, Dimitris Metaxas, and Chao Chen.
\newblock A topological filter for learning with label noise.
\newblock {\em Advances in neural information processing systems},
  33:21382--21393, 2020.

\bibitem{xiao2015learning}
Tong Xiao, Tian Xia, Yi~Yang, Chang Huang, and Xiaogang Wang.
\newblock Learning from massive noisy labeled data for image classification.
\newblock In {\em Proceedings of the IEEE conference on computer vision and
  pattern recognition}, pages 2691--2699, 2015.

\bibitem{chen2015webly}
Xinlei Chen and Abhinav Gupta.
\newblock Webly supervised learning of convolutional networks.
\newblock In {\em Proceedings of the IEEE international conference on computer
  vision}, pages 1431--1439, 2015.

\bibitem{he2022robust}
Sihong He, Yue Wang, Shuo Han, Shaofeng Zou, and Fei Miao.
\newblock A robust and constrained multi-agent reinforcement learning framework
  for electric vehicle amod systems.
\newblock {\em arXiv preprint arXiv:2209.08230}, 2022.

\bibitem{he2023robust}
Sihong He, Songyang Han, Sanbao Su, Shuo Han, Shaofeng Zou, and Fei Miao.
\newblock Robust multi-agent reinforcement learning with state uncertainties,
  2023.

\bibitem{shorten2019survey}
Connor Shorten and Taghi~M Khoshgoftaar.
\newblock A survey on image data augmentation for deep learning.
\newblock {\em Journal of big data}, 6(1):1--48, 2019.

\bibitem{krogh1991simple}
Anders Krogh and John Hertz.
\newblock A simple weight decay can improve generalization.
\newblock {\em Advances in neural information processing systems}, 4, 1991.

\bibitem{manwani2013noise}
Naresh Manwani and PS~Sastry.
\newblock Noise tolerance under risk minimization.
\newblock {\em IEEE transactions on cybernetics}, 43(3):1146--1151, 2013.

\bibitem{ghosh2017robust}
Aritra Ghosh, Himanshu Kumar, and P~Shanti Sastry.
\newblock Robust loss functions under label noise for deep neural networks.
\newblock In {\em Proceedings of the AAAI conference on artificial
  intelligence}, volume~31, 2017.

\bibitem{chen2021towards}
Zhuotong Chen, Qianxiao Li, and Zheng Zhang.
\newblock Towards robust neural networks via close-loop control.
\newblock {\em arXiv preprint arXiv:2102.01862}, 2021.

\bibitem{chen2022self}
Zhuotong Chen, Qianxiao Li, and Zheng Zhang.
\newblock Self-healing robust neural networks via closed-loop control.
\newblock {\em Journal of Machine Learning Research}, 23(319):1--54, 2022.

\bibitem{natarajan2013learning}
Nagarajan Natarajan, Inderjit~S Dhillon, Pradeep~K Ravikumar, and Ambuj Tewari.
\newblock Learning with noisy labels.
\newblock {\em Advances in neural information processing systems}, 26, 2013.

\bibitem{xu2009does}
Zong-Ben Xu, Rui Zhang, and Wen-Feng Jing.
\newblock When does online bp training converge?
\newblock {\em IEEE Transactions on Neural Networks}, 20(10):1529--1539, 2009.

\bibitem{sharkawy2018variable}
Abdel-Nasser Sharkawy, Panagiotis~N Koustournpardis, and Nikos Aspragathos.
\newblock Variable admittance control for human-robot collaboration based on
  online neural network training.
\newblock In {\em 2018 IEEE/RSJ International Conference on Intelligent Robots
  and Systems (IROS)}, pages 1334--1339. IEEE, 2018.

\bibitem{he2020data}
Sihong He, Lynn Pepin, Guang Wang, Desheng Zhang, and Fei Miao.
\newblock Data-driven distributionally robust electric vehicle balancing for
  mobility-on-demand systems under demand and supply uncertainties.
\newblock In {\em 2020 IEEE/RSJ International Conference on Intelligent Robots
  and Systems (IROS)}, pages 2165--2172. IEEE, 2020.

\bibitem{he2023data}
Sihong He, Zhili Zhang, Shuo Han, Lynn Pepin, Guang Wang, Desheng Zhang, John~A
  Stankovic, and Fei Miao.
\newblock Data-driven distributionally robust electric vehicle balancing for
  autonomous mobility-on-demand systems under demand and supply uncertainties.
\newblock {\em IEEE Transactions on Intelligent Transportation Systems}, 2023.

\bibitem{rousseeuw1984least}
Peter~J Rousseeuw.
\newblock Least median of squares regression.
\newblock {\em Journal of the American statistical association},
  79(388):871--880, 1984.

\bibitem{yu2018efficient}
Xiyu Yu, Tongliang Liu, Mingming Gong, Kayhan Batmanghelich, and Dacheng Tao.
\newblock An efficient and provable approach for mixture proportion estimation
  using linear independence assumption.
\newblock In {\em Proceedings of the IEEE Conference on Computer Vision and
  Pattern Recognition}, pages 4480--4489, 2018.

\bibitem{northcutt2021confident}
Curtis Northcutt, Lu~Jiang, and Isaac Chuang.
\newblock Confident learning: Estimating uncertainty in dataset labels.
\newblock {\em Journal of Artificial Intelligence Research}, 70:1373--1411,
  2021.

\bibitem{hazan_efficient_2017}
Elad Hazan, Karan Singh, and Cyril Zhang.
\newblock Efficient {Regret} {Minimization} in {Non}-{Convex} {Games}, July
  2017.
\newblock arXiv:1708.00075 [cs, stat].

\bibitem{yang4040704vflh}
Yifan Yang, Lin Chen, Pan Zhou, and Xiaofeng Ding.
\newblock Vflh: A following-the-leader-history based algorithm for adaptive
  online convex optimization with stochastic constraints.
\newblock {\em Available at SSRN 4040704}, 2022.

\bibitem{yu2022mr}
Xinling Yu, Jos{\'e}~EC Serrall{\'e}s, Ilias~I Giannakopoulos, Ziyue Liu, Luca
  Daniel, Riccardo Lattanzi, and Zheng Zhang.
\newblock Mr-based electrical property reconstruction using physics-informed
  neural networks.
\newblock {\em arXiv preprint arXiv:2210.12584}, 2022.

\bibitem{yu2023pifon}
Xinling Yu, Jos{\'e}~EC Serrall{\'e}s, Ilias~I Giannakopoulos, Ziyue Liu, Luca
  Daniel, Riccardo Lattanzi, and Zheng Zhang.
\newblock Pifon-ept: Mr-based electrical property tomography using
  physics-informed fourier networks.
\newblock {\em arXiv preprint arXiv:2302.11883}, 2023.

\bibitem{yu2017online}
Hao Yu, Michael Neely, and Xiaohan Wei.
\newblock Online convex optimization with stochastic constraints.
\newblock {\em Advances in Neural Information Processing Systems}, 30, 2017.

\end{thebibliography}
\end{document}